\documentclass[letterpaper,10pt,conference,twocolumn]{ieeeconf} 
\IEEEoverridecommandlockouts

\usepackage{verbatim}
\usepackage[pdftex]{graphicx}
\usepackage{epsfig,subfigure}
\usepackage{amsmath,amssymb,amsfonts,bm,amscd} 
\usepackage{mathrsfs}
\usepackage{setspace}
\usepackage{fancyhdr}
\usepackage{algorithm} 
\usepackage{psfrag}
\usepackage{cite}
\usepackage{makecell}
\usepackage{url}

\newtheorem{theorem}{Theorem}[section]

\newtheorem{lemma}[theorem]{Lemma}

\newtheorem{assumption}{Assumption}[section]
\newtheorem{remark}{Remark}[section]

\newcommand{\NN}{{\mathbb N}}
\newcommand{\RR}{{\mathbb R}}

\newcommand{\cA}{{\mathcal A}}

\title{\bf Provably Robust Learning-Based Approach for High-Accuracy Tracking Control of Lagrangian Systems
\thanks{The authors are with the Dynamic Systems Lab (www.dynsyslab.org), Institute
for Aerospace Studies, University of Toronto, Canada. E-mail:
mohamed.helwa@robotics.utias.utoronto.ca, adam.heins@robotics.utias.utoronto.ca, schoellig@utias.utoronto.ca.
~~~~This research was supported by NSERC grant
RGPIN-2014-04634 and OCE/SOSCIP TalentEdge Project \#27901.}}%NSERC grant RGPIN-2014-04634 and the Connaught New Researcher Award.}} 
\author{Mohamed~K.~Helwa, Adam Heins, and Angela P. Schoellig}% <-this % stops a space

\begin{document}
\maketitle

\begin{abstract}
Lagrangian systems represent a wide range of robotic systems, including manipulators, wheeled and legged robots, and quadrotors. Inverse dynamics control and feedforward linearization techniques are typically used to convert the complex nonlinear dynamics of Lagrangian systems to a set of decoupled double integrators, and then a standard, outer-loop controller can be used to calculate the commanded acceleration for the linearized system. However, these methods typically depend on having a very accurate system model, which is often not available in practice. While this challenge has been addressed in the literature using different learning approaches, most of these approaches do not provide safety guarantees in terms of stability of the learning-based control system. In this paper, we provide a novel, learning-based control approach based on Gaussian processes (GPs) that ensures both stability of the closed-loop system and high-accuracy tracking. We use GPs to approximate the error between the commanded acceleration and the actual acceleration of the system, and then use the predicted mean and variance of the GP to calculate an upper bound on the uncertainty of the linearized model. This uncertainty bound is then used in a robust, outer-loop controller to ensure stability of the overall system. Moreover, we show that the tracking error converges to a ball with a radius that can be made arbitrarily small. Furthermore, we verify the effectiveness of our approach via simulations on a 2 degree-of-freedom (DOF) planar manipulator and experimentally on a 6~DOF industrial manipulator.

\end{abstract}

\section{Introduction}
\label{sec:introd}
High-accuracy tracking is an essential requirement in advanced manufacturing, self-driving cars, medical robots, and autonomous flying vehicles, among others. To achieve high-accuracy tracking for these complex, typically high-dimensional, nonlinear robotic systems, a standard approach is to use inverse dynamics control \cite{spong} or feedforward linearization techniques \cite{fflin} to convert the complex nonlinear dynamics into a set of decoupled double integrators. Then, a standard, linear, outer-loop controller, e.g., a proportional-derivative (PD) controller, can be used to make the decoupled linear system track the desired trajectory \cite{spong}. However, these linearization techniques depend on having accurate system models, which are difficult to obtain in practice.%from the first principles of physics.

\begin{figure}[t]
\begin{center}
\includegraphics[scale=.28, trim = 10mm 20mm 10mm 10mm]{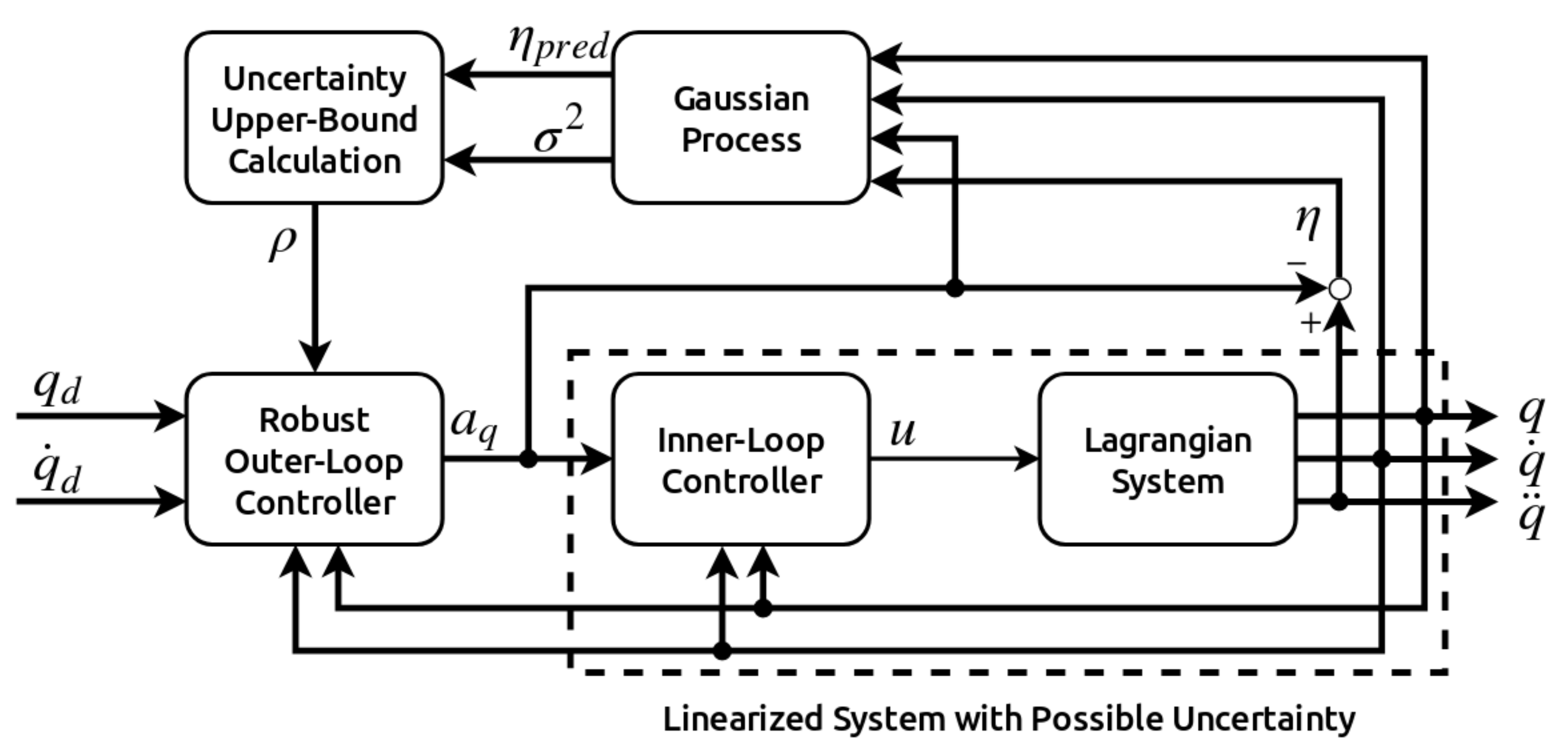} 
\end{center}
\caption{Block diagram of our proposed strategy. GP regression models learn the uncertainty in the linearized model. Using the mean and variance of the GP, one can calculate an upper bound on the uncertainty to be used in a robust, outer-loop controller. The symbol $q$ is the actual position vector, $q_d$ is the desired position vector, $a_q$ is the commanded acceleration vector, $u$ is the force/torque input vector, and $\eta$ is the uncertainty vector.}
\label{fig:structure}
\end{figure}

To address this problem, robust control techniques have been used for many decades to design the outer-loop controllers to account for the uncertainties in the model~\cite{robust_con}. However, the selection of the uncertainty bounds in the robust controller design is challenging. On the one hand, selecting high bounds typically results in a conservative behavior, and hence, a large tracking error. On the other hand, relatively small uncertainty bounds may not represent the true upper bounds of the uncertainties, and consequently, stability of the overall system is not ensured. Alternatively, several approaches have been proposed for learning the inverse system dynamics from collected data where the system models are not available or not sufficiently accurate; see \cite{thesis,peters1,invlearn1,invlearn2}. Combining a-priori model knowledge with learning data has also been studied in \cite{thesis,invlearn_model1}. However, these learning approaches typically neglect the learning regression errors in the analysis, and they do not provide a proof of stability of the overall, learning-based control system, which is crucial for safety-critical applications such as medical robots. The limitations of the robust control and the learning-based techniques show the urgent need for novel, robust, learning-based control approaches that ensure both stability of the overall control system and high-accuracy tracking. This sets the stage for the research carried out in this paper.  

In this paper, we provide a novel, robust, learning-based control technique that achieves both closed-loop stability and high-accuracy tracking. In particular, we use Gaussian processes (GPs) to approximate the error between the commanded acceleration to the linearized system and the actual acceleration of the robotic system, and then use the predicted mean and variance of the GP to calculate an upper bound on the uncertainty of the linearization. This uncertainty bound is then used in a robust, outer-loop controller to ensure stability of the overall system (see Figure~\ref{fig:structure}). Moreover, we show that using our proposed strategy, the tracking error converges to a ball with a radius that can be made arbitrarily small through appropriate control design, and hence, our proposed approach also achieves high-accuracy tracking. Furthermore, we verify the effectiveness of the proposed approach via simulations on a 2 DOF planar manipulator using MATLAB Simulink and experimentally on a UR10 6 DOF industrial manipulator.%, and \emph{(ii)} experimental results on a 6 DOF industrial manipulator. Our approach is different from \cite{Kulic} from several aspects.

This paper is organized as follows. Section~\ref{sec:related} provides a summary of some recent related work.
%, while Section \ref{sec:background} provides preliminary, relevant background. 
Section~\ref{sec:problem} describes the considered problem, and Section~\ref{sec:alg} provides the proposed approach. Section~\ref{sec:theory} derives theoretical guarantees for the proposed approach. Section~\ref{sec:sim} and \ref{sec:exp} include the simulation and experimental results, and Section~\ref{sec:con} concludes the paper.      
%%that are necessary to understand the rest of the paper. 
%In Section \ref{sec:problem}, we define the transfer learning problem studied in this paper. In Section \ref{sec:basic}, we provide theoretical results on transformation maps that achieve perfect transfer learning, and then utilize these results to provide insights into the correct features of optimal transfer maps. In Section \ref{sec:alg}, we present our proposed, practical algorithm. Section \ref{sec:examples} includes two robotic applications, and Section \ref{sec:con} concludes the paper. 

\emph{Notation and Basic Definitions:} For a set $S$, $\bar{S}$ denotes its closure and $S^{\circ}$ its interior. The notation $B_{\delta}(y)$ denotes a ball of radius $\delta$ centered at a point $y$. A matrix $P$ is \emph{positive definite} if it is symmetric and all its eigenvalues are positive. For a vector $x$, $\|x\|$ denotes its Euclidean norm. A function $f(x)$ is \emph{smooth} if its partial derivatives of all orders exist and are continuous. The solutions of $\dot{x}=f(t,x)$ are \emph{uniformly ultimately bounded with ultimate bound $b$} if there exist positive constants $b,~c$, and for every $0<a<c$, there exists $T(a,b)\geq 0$ such that $\|x(t_0)\|\leq a$ implies $\|x(t)\|\leq b$, for all $t\geq (T+t_0)$, where $t_0$ is the initial time instant. A \emph{kernel} is a symmetric function $k:\cA\times \cA\rightarrow \RR$. A \emph{reproducing kernel
Hilbert space (RKHS)} corresponding to a kernel $k(.,.)$ includes functions of the form $f(a)=\sum_{j=1}^{m}\alpha_jk(a,a_j)$ with $m\in \NN$, $\alpha_j\in \RR$ and representing points $a_j\in \cA$. 
%The RKHS norm, denoted $\|f\|_k$, is defined based on the definition of inner product of functions, and its boundedness implies that the coefficients $\alpha_j$ decay sufficiently fast as $j$ increases, and the functions are regular with respect to the kernel.  

%We say that a function is a kernel %We also call $B_b(0)$ the ultimate ball.  
  
\section{Related Work}
\label{sec:related}

The study of safe learning dates back to the beginning of this century \cite{general_safe_learning}. In \cite{LyapRL1} and \cite{LyapRL2}, Lyapunov-based reinforcement learning is used to allow a learning agent to safely switch between pre-computed baseline controllers. Then, in \cite{risk_rl}, risk-sensitive reinforcement learning is proposed, in which the expected return is heuristically weighted with the probability of reaching an error state. In several other papers, including \cite{safe_exp1}, \cite{safe_exp2} and \cite{GP_felix}, safe exploration methods are utilized to allow the learning modules to achieve a desired balance between ensuring safe operation and exploring new states for improved performance. In \cite{general_safe_learning}, a general framework is proposed for ensuring safety of learning-based control strategies for uncertain robotic systems. In this framework, robust reachability guarantees from control theory are combined with Bayesian analysis based on empirical observations. The result is a safety-preserving, supervisory controller of the learning module that allows the system to freely execute its learning policy almost everywhere, but imposes control actions to ensure safety at critical states. Despite its effectiveness for ensuring safety, the supervisory controller in this approach has no role in reducing tracking errors.  %that can be caused by the regression errors of the learning modules.  

%Focusing our attention on safe, learning-based inverse dynamics control, we refer to \cite{Schaal,Kulic}. In these papers, Gaussian processes (GPs) are utilized to learn the errors in the output torques of the inverse dynamics model online. In \cite{Schaal}, the GP learning is combined with a state-of-the-art gradient descent method for learning feedback terms online. The main idea behind this approach is that the offsets learnt online by the gradient descent method would correct for fast perturbations, while the GP is responsible for correcting slow perturbations. This allows for exponential smoothing of the GP hyperparameters, which increases the robustness of the GP at the cost of having slower reactiveness. Nevertheless, \cite{Schaal} does not provide a proof of the robust stability of the closed-loop system. In \cite{Kulic}, the variance of the GP prediction is utilized to adapt the parameters of an outer-loop PD controller online, and the uniform ultimate boundedness of the tracking error is proved under some assumptions on the structure of the PD controller (e.g., the gain matrix was assumed to be diagonal, which imposes a decentralized gain control scheme). The results of \cite{Kulic} are verified via simulations on a 2 DOF manipulator.

Focusing our attention on safe, learning-based inverse dynamics control, we refer to \cite{adaptive1,Schaal,Kulic}. In \cite{adaptive1}, a model reference adaptive control (MRAC) architecture based on Gaussian processes (GPs) is proposed, and stability of the overall control system is proved. While the approach in \cite{adaptive1} is based on adaptive control theory, our approach is based on robust control theory. In particular, in \cite{adaptive1}, the mean of the GP is used to exactly cancel the uncertainty vector, while in our approach, we use both the mean and variance of the GP to learn an upper bound on the uncertainty vector to be used in a robust, outer-loop controller. Hence, unlike \cite{adaptive1}, in our approach, the uncertainty of the learning module is not only incorporated in the stability analysis but also in the outer-loop controller design. Intuitively, the less certain our GPs are, the more robust the outer-loop controller should be for ensuring safety. When more data is collected and the GPs are more certain, the outer-loop controller can be less conservative for improved performance. While the results of \cite{adaptive1} are tested in simulations on a two-dimensional system, we test our results experimentally on a 6 DOF manipulator. 

In \cite{Schaal,Kulic}, GPs are utilized to learn the errors in the output torques of the inverse dynamics model online. In \cite{Schaal}, the GP learning is combined with a state-of-the-art gradient descent method for learning feedback terms online. The main idea behind this approach is that the gradient descent method would correct for fast perturbations, while the GP is responsible for correcting slow perturbations. This allows for exponential smoothing of the GP hyperparameters, which increases the robustness of the GP at the cost of having slower reactiveness. Nevertheless,~\cite{Schaal} does not provide a proof of the robust stability of the closed-loop system. In~\cite{Kulic}, the variance of the GP prediction is utilized to adapt the parameters of an outer-loop PD controller online, and the uniform ultimate boundedness of the tracking error is proved under some assumptions on the structure of the PD controller (e.g., the gain matrix was assumed to be diagonal, which imposes a decentralized gain control scheme). The results of~\cite{Kulic} are verified via simulations on a 2 DOF manipulator.

Our approach differs from \cite{Kulic} in several aspects. First, we do not use an adaptive PD controller in the outer loop, but add a robustness term to the output of the outer-loop controller. Second, while \cite{Kulic} uses the GP to learn the error in the estimated torque from the nominal inverse dynamics, in our approach, we learn the error between the commanded and actual accelerations. This can be beneficial in two aspects: \emph{(i)} This makes our approach applicable to industrial manipulators that have onboard controllers for calculating the torque and allow the user to only send commanded acceleration/velocity; 
%We will show that benefit in detail in our simulation and experimental results in Sections \ref{} and \ref{}; 
\emph{(ii)} this makes our approach also applicable beyond inverse dynamics control of manipulators; indeed, our proposed approach can be applied to any Lagrangian system for which feedforward/feedback linearization can be used to convert the nonlinear dynamics of the system to a set of decoupled double integrators, such as a quadrotor under a feedforward linearization, see Section 5.3 of \cite{HS17}. Third, while \cite{Kulic} shows uniform ultimate boundedness of the tracking error, it does not provide discussions on the size of the ultimate ball. In this work, we show that using our proposed approach, the size of the ball can be made arbitrarily small through the control design. Fourth, in our approach, we do not impose any assumption on the structure of the outer-loop PD controller and decentralized, outer-loop control is not needed for our proof. Finally, we verify our approach experimentally on a 6 DOF manipulator.

\section{Problem Statement}
\label{sec:problem}
In this paper, we consider Lagrangian systems, which represent a wide class of mechanical systems \cite{Murray}. In what follows, we focus our attention on a class of Lagrangian systems represented by:
\begin{equation}
\label{eq:sys_man}
M(q(t))\ddot{q}(t)+C(q(t),\dot{q}(t))\dot{q}(t)+g(q(t))=u(t),
\end{equation}
where $q=(q_1,\cdots,q_N)$ is the vector of generalized coordinates (displacements or angles), $\dot{q}=(\dot{q}_1,\cdots,\dot{q}_N)$ is the vector of generalized velocities, $u=(u_1,\cdots,u_N)$ is the vector of generalized forces (forces or torques), $N$ is the system's degree of freedom, $M$, $C$, and $g$ are matrices of proper dimensions and smooth functions, and $M(q)$ is a positive definite matrix. Fully-actuated robotic manipulators are an example of Lagrangian systems that can be expressed by~\eqref{eq:sys_man}. Despite focusing our discussion on systems represented by~\eqref{eq:sys_man}, we emphasize that our results in this paper can be easily generalized to a wider class of nonlinear Lagrangian systems for which feedback/feedforward linearization can be utilized to convert the dynamics of the system into a set of decoupled double integrators plus an uncertainty vector. 
%as we will discuss in detail in Remark \ref{rem:general}.

For the nonlinear system \eqref{eq:sys_man} with uncertain matrices $M$, $C$, and $g$, we aim to make the system positions and velocities $(q(t),\dot{q}(t))$ track a desired smooth trajectory $(q_d(t),\dot{q}_d(t))$. For simplicity of notation, in our discussion, we drop the dependency on time t from $q,~q_d$, their derivatives, and $u$. Our goal is to design a novel, learning-based control strategy that is easy to interpret and implement, and that satisfies the following desired objectives:
\begin{itemize}
\item [(O1)] \emph{Robustness:} The overall, closed-loop control system satisfies robust stability in the sense that the tracking error has an upper bound under the system uncertainties. 
\item [(O2)] \emph{High-Accuracy Tracking:} For feasible desired trajectories, the tracking error converges to a ball around the origin that can be made arbitrarily small through the control design. 
%for feasible desired trajectories that can be followed with inputs within the actuation limits. 
For the ideal case, where the pre-assumed system parameters are correct, the tracking error should converge exponentially to the origin.
\item [(O3)] \emph{Adaptability:} The proposed strategy should incorporate online learning to continuously adapt to online changes of the system parameters and disturbances.  
\item [(O4)] \emph{Generalizability of the Approach:} The proposed approach should be general enough to be also applicable to industrial robots that have onboard controllers for calculating the forces/torques and allow the user to send only commanded acceleration/velocity. %It is also desirable that the approach is general enough to be implementable on other Lagrangian systems beyond \eqref{eq:sys_man}. 
\end{itemize}

\section{Methodology}
\label{sec:alg}
We present our proposed methodology, and then in the next sections, we show that it satisfies objectives (O1)-(O4).
%In this section, we start by presenting some existing approaches for solving the tracking problem for \eqref{eq:sys_man} and their limitations. Inspired by this discussion, we then present our proposed methodology, and prove that it satisfies objectives (O1)-(O2). We then provide detailed remarks on how our proposed approach also satisfies objectives (O3)-(O4). 

A standard approach for solving the tracking control problem for \eqref{eq:sys_man} is inverse dynamics control. Since $M(q)$ is positive definite by assumption, it is invertible. Hence, it is evident that if the matrices $M$, $C$, and $g$ are all known, then the following inverse dynamics control law
\begin{equation}
\label{eq:inv_ideal}
u=C(q,\dot{q})\dot{q}+g(q)+M(q)a_q
\end{equation}
converts the complex nonlinear dynamic system \eqref{eq:sys_man} into 
\begin{equation}
\label{eq:lin_sys_ideal}
\ddot{q}=a_q,
\end{equation} 
where $a_q$ is the commanded acceleration, a new input to the linearized system \eqref{eq:lin_sys_ideal} to be calculated by an outer-loop control law, e.g., a PD controller (see Figure \ref{fig:structure}). However, the standard inverse dynamics control \eqref{eq:inv_ideal} heavily depends on accurate knowledge of the system parameters. In practice, the matrices $M$, $C$, and $g$ are not perfectly known, and consequently, one has to use estimated values of these matrices $\hat{M}$, $\hat{C}$, and $\hat{g}$, respectively, where $\hat{M}$, $\hat{C}$, and $\hat{g}$ are composed of smooth functions. Hence, in practice, the control law \eqref{eq:inv_ideal} should be replaced with
\begin{equation}
\label{eq:inv_prac}
u=\hat{C}(q,\dot{q})\dot{q}+\hat{g}(q)+\hat{M}(q)a_q.
\end{equation}     
Now by plugging \eqref{eq:inv_prac} into the system model \eqref{eq:sys_man}, we get
\begin{equation}
\label{eq:lin_sys_prac}
\ddot{q}=a_q+\eta(q,\dot{q},a_q),
\end{equation}   
where $\eta(q,\dot{q},a_q)=M^{-1}(q)(\tilde{M}(q)a_q+\tilde{C}(q,\dot{q})\dot{q}+\tilde{g}(q))$, with $\tilde{M}=\hat{M}-M$, $\tilde{C}=\hat{C}-C$, and $\tilde{g}=\hat{g}-g$. 
It can be shown that even if the left hand side (LHS) of~\eqref{eq:sys_man} has a smooth, unstructured, added uncertainty $E(q,\dot{q})$, e.g., unmodeled friction, \eqref{eq:lin_sys_prac} is still valid with modified~$\eta$.   
%\footnote{For system \eqref{eq:sys_man}, we derived that the uncertainty $\eta$ resulting from the inverse dynamics control is an additive vector. If for another system, one has a multiplicative uncertainty, i.e., $\ddot{q}=\eta_m(q,\dot{q},a_q)a_q$, where $\eta_m$ is a scalar function, then one can write $\ddot{q}=a_q+\eta(q,\dot{q},a_q)$, where $\eta(q,\dot{q},a_q)=(-1+\eta_m(q,\dot{q},a_q))a_q$, and still use the results of this paper.}. 
Because of $\eta$, the dynamics \eqref{eq:lin_sys_prac} resulting from the inverse dynamics control are still nonlinear and coupled. To control the uncertain system \eqref{eq:lin_sys_prac}, on the one hand, robust control methods are typically very conservative, while on the other hand, learning methods do not provide stability guarantees.

Hence, in this paper, we combine ideas from robust control theory with ideas from machine learning, particularly Gaussian processes (GPs) for regression, to provide a robust, learning-based control strategy that satisfies objectives (O1)-(O4). The main idea behind our proposed approach is to use GPs to learn the uncertainty vector $\eta(q,\dot{q},a_q)$ in \eqref{eq:lin_sys_prac} online. Following \cite{Kulic}, we use a set of $N$ independent GPs, one for learning each element of $\eta$, to reduce the complexity of the regression. It is evident that conditioned on knowing $q,~\dot{q}$, and $a_q$, one can learn each element of $\eta$ independently from the rest of the elements of $\eta$. A main advantage of GP regression is that it does not only provide an estimated value of the mean $\mu$, but it also provides an expected variance $\sigma^2$, which represents the accuracy of the regression model based on the distance to the training data. The punchline here is that one can use both the mean and variance of the GP to calculate an upper bound $\rho$ on~$\|\eta\|$ that is guaranteed to be correct with high probability, as we will show later in this section. One can then use this upper bound to design a robust, outer-loop controller that ensures robust stability of the overall system. Hence, our proposed strategy consists of three parts: 
%\begin{itemize}

{\bf \emph{(i)} \emph{Inner-Loop Controller:}} We use the inverse dynamics control law \eqref{eq:inv_prac}, where $\hat{M}$, $\hat{C}$, and $\hat{g}$ are estimated values of the system matrices from an a-priori model.

{\bf \emph{(ii)} \emph{GPs for Learning the Uncertainty:}} We use a set of $N$ GPs to learn the uncertainty vector $\eta$ in \eqref{eq:lin_sys_prac}. We start by reviewing GP regression \cite{GP_conf,GP_felix}. A GP is a nonparametric regression model that is used to approximate a nonlinear function $J(a):\cA\rightarrow \RR$, where $a\in \cA$ is the input vector. The ability of the GP to approximate the function is based on the assumptions that function values $J(a)$ associated with different values of $a$ are random variables, and that any finite number of these variables have a joint Gaussian distribution. The GP predicts the value of the function, $J(a^*)$, at an arbitrary input $a^*\in \cA$ from a set of $n$ observations $D_{n}:=\{a_j,\hat{J}(a_j)\}_{j=1}^{n}$, where $\hat{J}(a_j)$, $j\in\{1,\cdots,n\}$, are assumed to be noisy measurements of the function's true values. That is, $\hat{J}(a_j)=J(a_j)+\omega'$, where $\omega'$ is a zero mean Gaussian noise with variance $\sigma_\omega^2$. Assuming, without loss of generality (w.l.o.g.), a zero prior mean of the GP and conditioned on the previous observations, the mean and variance of the GP prediction are given by: 
\begin{equation}
\label{eq:GP_mean}
\mu_n(a^*)={\bf k}_n(a^*)({\bf K}_n+{\bf I}_n\sigma_\omega^2)^{-1}{\bf \hat{J}}_n,
\end{equation} 
\begin{equation}
\label{eq:GP_variance}
\sigma_n^2(a^*)=k(a^*,a^*)-{\bf k}_n(a^*)({\bf K}_n+{\bf I}_n\sigma_\omega^2)^{-1}{\bf k}_n^T(a^*),
\end{equation}
respectively, where ${\bf \hat{J}}_n=[\hat{J}(a_1),\cdots,\hat{J}(a_n)]^T$ is the vector of observed, noisy function values. The matrix ${\bf K}_n\in \RR^{n\times n}$ is the covariance matrix with entries $[{\bf K}_n]_{(i,j)}=k(a_i,a_j)$, $i,~j\in\{1,\cdots,n\}$, where $k(a_i,a_j)$ is the covariance function defining the covariance between two function values $J(a_i),~J(a_j)$ (also called the kernel). The vector ${\bf k}_n(a^*)=[k(a^*,a_1),\cdots,k(a^*,a_n)]$ contains the covariances between the new input and the observed data points, and ${\bf I}_n\in \RR^{n\times n}$ is the identity matrix.     
The tuning of the GP is typically done through the selection of the kernel function and the tuning of its hyperparameters. For information about different standard kernel functions, please refer to \cite{GP_conf}.  
%For the GP, we assume without loss of generality (w.l.o.g.) that it has a zero a-priori mean. Then, we define

We next discuss our implementation of the GPs. The GPs run in discrete time with sampling interval $T_s$. At a sampling instant $k$, the inputs to each GP regression model are the same $(q(k),\dot{q}(k),{a_q}(k))$, and the output is an estimated value of an element of the $\eta$ vector at $k$. For the training data for each GP, $n$ observations of $(q,\dot{q},a_q)$ are used as the labeled input together with $n$ observations of an element of the vector $\ddot{q}-a_q+{\bf \omega}_v$ as the labeled output, where ${\bf \omega}_v\in \RR^N$ is Gaussian noise with zero mean and variance $diag(\sigma_{\omega_1}^2,\cdots,\sigma_{\omega_N}^2)$; see~\eqref{eq:lin_sys_prac}. For selecting the $n$ observations, we use the oldest point (OP) scheme for simplicity; this scheme depends on removing the oldest observation to accommodate for a new one \cite{adaptive1}. We use the squared exponential kernel 
\begin{equation}
\label{eq:GP_kernel1}
k(a_i,a_j)=\sigma_{\eta}^2e^{-\frac{1}{2}(a_i-a_j)^T\bar{M}^{-2}(a_i-a_j)},%(1+\sqrt{3}r(a_i,a_j))e^{-\sqrt{3}r(a_i,a_j)},
\end{equation} 
%\begin{equation}
%\label{eq:GP_kernel2}
%r(a_i,a_j)=\sqrt{(a_i-a_j)^TM^{-2}(a_i-a_j)},
%\end{equation}
which is parameterized by the hyperparameters: $\sigma_\eta^2$, the prior variance, and the positive length scales $l_1,\cdots,l_{3N}$ which are the diagonal elements of the diagonal matrix $\bar{M}$. Hence, the expected mean and variance of each GP can be obtained through equations \eqref{eq:GP_mean}-\eqref{eq:GP_kernel1}. Guidelines for tuning the GP hyperparameters $\sigma_\omega^2,\sigma_\eta^2,l_1,\cdots,l_{3N}$ can be found in \cite{GP_felix}. 

As stated before, a main advantage of GP regression is that the GP provides a variance, which represents the accuracy of the regression model based on the distance between the new input and the training data. 
One can then use the predicted mean and variance of the GP to provide a confidence interval around the mean that is guaranteed to be correct with high probability. There are several comprehensive studies in the machine learning literature on calculating these confidence intervals.
%Unfortunately, it is a well-known fact in the GP regression that with iteration-independent confidence intervals, e.g., $[\mu_n-3\sigma_n,\mu_n+3\sigma_n]$, the probability guarantee will rapidly drop to zero as the GP run time increases \cite{GP_conf}. Instead, there are many deep studies in the machine learning literature on using iteration-dependent size of the confidence interval that increases with run time very slowly, and is proved to be correct with high probability for all sampling instants. 
For completeness, we review one of these results, particularly Theorem 6 of \cite{GP_conf}. Let $a_{aug}=(q,\dot{q},a_q)$, and $\eta_i(a_{aug})$ denote the $i$-th element of the unknown vector $\eta$. 

\begin{assumption}
\label{assum0}
The function $\eta_i(a_{aug})$, $i\in\{1,\cdots,N\}$, has a bounded RKHS norm $\|\eta_i\|_{k}$ with respect to the covariance function $k(a,a')$ of the GP, and the noise $\omega_i$ added to the output observations, $i\in\{1,\cdots,N\}$, is uniformly bounded by $\bar{\sigma}$.  
\end{assumption}     

The RKHS norm is a measure of the function smoothness, and its boundedness implies that the function is well-behaved in the sense that it is regular with respect to the kernel \cite{GP_conf}. Intuitively, Assumption \ref{assum0} does not hold if the uncertainty~$\eta$ is discontinuous, e.g., discontinuous friction. 

\begin{lemma}[Theorem 6 of \cite{GP_conf}]
\label{lemma1}
Suppose that Assumption \ref{assum0} holds. Let $\delta_{p}\in (0,1)$. Then, $\Pr\{\forall a_{aug} \in \cA,\|\mu(a_{aug})-\eta_i(a_{aug})\|\leq \beta^{1/2}\sigma(a_{aug})\}\geq 1-\delta_p$,     
where $\Pr$ stands for the probability, $\cA\subset \RR^{3N}$ is compact, $\mu(a_{aug}),~\sigma^2(a_{aug})$ are the GP mean and variance evaluated at $a_{aug}$ conditioned on $n$ past observations, and $\beta=2\|\eta_i\|_{k}^2+300\gamma \ln^3((n+1)/{\delta}_p)$. The variable $\gamma\in \RR$ is the maximum information gain and is given by $\gamma=\max_{\{a_{aug,1},\cdots,a_{aug,n+1}\}\in \cA}0.5~\log(\det(I+\bar{\sigma}^{-2}{\bf K}_{n+1}))$, where $\det$ is the matrix determinant, $I\in \RR^{(n+1)\times (n+1)}$ is the identity matrix, ${\bf K}_{n+1}\in \RR^{(n+1)\times (n+1)}$ is the covariance matrix given by $[{\bf K}_{n+1}]_{(i,j)}=k(a_{aug,i},a_{aug,j})$, $i,~j\in\{1,\cdots,n+1\}$.  
%${\bf K}_{n+1}=[k(a,a')]_{a,a'\in\{a_{aug,1},\cdots,a_{aug,n+1}\}}$, 
%$and $a_{aug,n+1}=a_{aug}$.
\end{lemma}

Finding the information gain maximizer can be approximated by an efficient greedy algorithm \cite{GP_conf}. Indeed, $\gamma$ has a sublinear dependence on $n$ for many commonly used kernels, and can be numerically approximated by a constant \cite{Kulic}. 

The punchline here is that we know from Lemma \ref{lemma1} that one can define for each GP a confidence interval around the mean that is guaranteed to be correct for all points $a_{aug}\in \cA$, a compact set, with probability higher than $(1-\delta_p)$, where $\delta_p$ is typically picked very small. Let $\mu_{k,i}$ and $\sigma_{k,i}^2$ represent the expected mean and variance of the $i$-th GP at the sampling instant $k$, respectively, and let $\beta_i$ denote the $\beta$ parameter in Lemma \ref{lemma1} of the $i$-th GP, where $i\in\{1,\cdots,N\}$. 
%Let $\mu_{k,i}$, $\sigma_{k,i}^2$, and $\beta_{k,i}$ represent the expected mean, variance, and $\beta$ parameter in Lemma \ref{lemma1} of the $i$-th GP at the sampling instant $k$, respectively, where $i\in\{1,\cdots,N\}$. 
We select the upper bound on the absolute value of $\eta_i$ at $k$ to be  
\begin{equation}
\label{eq_rho_i}
\rho_{k,i}(\mu_{k,i},\sigma_{k,i})=\max(|\mu_{k,i}-\beta_i^{1/2}\sigma_{k,i}|,|\mu_{k,i}+\beta_i^{1/2}\sigma_{k,i}|). 
\end{equation}
%A main advantage of the GP is that it ensures that the value of the unknown element $\eta_i(k)$ of the uncertainty vector $\eta(k)$, $i\in\{1,\cdots,N\}$, lies in the confidence region $[\mu_i(k)-n\sigma_i(k),\mu_i(k)+n\sigma_i(k)]$ with high probability $\beta$, where $\beta=erf(n/\sqrt{2})$ and $erf(x):=(2/\sqrt{\pi})\int_{0}^{x}e^{-s^2}ds$. For instance, for $n=3$, the value of $\eta_i(k)$ lies in the confidence region $[\mu_i(k)-3\sigma_i(k),\mu_i(k)+3\sigma_i(k)]$ with probability $99.73/\%$, and so on. Hence, we select the upper bound on each element of $\eta(k)$, $\eta_i(k)$, to be 
%\begin{equation}
%\label{eq_rho_i}
%\rho_i(\mu_i(k),\sigma_i(k))=\max(|\mu_i(k)-n\sigma_i(k)|,|\mu_i(k)+n\sigma_i(k)|). 
%\end{equation}
Then, a good estimate of the upper bound on $\|\eta\|$ at $k$ is %the sampling instant $k$, $\rho_k$, is 
\begin{equation}
\label{eq_rho}
\rho_k=\sqrt{\rho_{k,1}(\mu_{k,1},\sigma_{k,1})^2+\cdots+\rho_{k,N}(\mu_{k,N},\sigma_{k,N})^2}.       
\end{equation}

%Note that although the value of the unknown function $\eta_i$ is guaranteed to lie in the confidence interval defined in Lemma \ref{lemma1} at all sampling instants with high probability, $\beta_k$ defining the confidence interval in Lemma \ref{lemma1} blows up as $k\rightarrow \infty$. From a practical perspective, this happens too slowly so that one can still learn the unknown function \cite{GP_conf}. However, this poses challenges on providing any asymptotic guarantees of the overall control system behavior as $k\rightarrow \infty$. To address this challenge, we propose in the next section a slightly-tailored version of the upper bound \eqref{eq_rho} that can be used in providing asymptotic guarantees of the system behavior.   

{\bf \emph{(iii)} \emph{Robust, Outer-Loop Controller:}} We use the estimated upper bound $\rho_k$ to design a robust, outer-loop controller. In particular, for a smooth, bounded desired trajectory $q_d(t)$, we use the outer-loop control law 
\begin{equation}
\label{eq:outer_controller}
a_q(t)=\ddot{q}_d(t)+K_P(q_d(t)-q(t))+K_D(\dot{q}_d(t)-\dot{q}(t))+r(t),
\end{equation} 
where $K_P\in \RR^{N\times N}$ and $K_D\in \RR^{N\times N}$ are the proportional and derivative matrices of the PD control law, respectively, and $r\in \RR^N$ is an added vector to the PD control law that will be designed to achieve robustness. Let $e(t):=(q(t)-q_d(t),\dot{q}(t)-\dot{q}_d(t))$ denote the tracking error vector. From \eqref{eq:outer_controller} and \eqref{eq:lin_sys_prac}, it can be shown that the tracking error dynamics are
\begin{equation}
\label{eq:tracking_error}
\dot{e}(t)=Ae(t)+B(r(t)+\eta(q(t),\dot{q}(t),a_q(t))),
\end{equation}
where 
\begin{equation}
\label{eq:AandB}
A=\left[ 
\begin{array}{cc}
0 & I  \\
-K_P & -K_D  
\end{array} 
\right]\in \RR^{2N\times2N},~~
B=\left[ 
\begin{array}{rr}
0 \\ I \end{array} 
\right]\in \RR^{2N\times N},
\end{equation}
 and $I\in \RR^{N\times N}$ is the identity matrix. From \eqref{eq:tracking_error} and \eqref{eq:AandB}, it is clear that the controller matrices $K_P$ and $K_D$ should be designed to make $A$ a Hurwitz matrix. 

We now discuss how to design the robustness vector $r(t)$. To that end, let $P\in \RR^{2N\times 2N}$ be the unique positive definite matrix satisfying $A^TP+PA=-Q$, where $Q\in \RR^{2N\times 2N}$ is a positive definite matrix. We define $r(t)$ as follows
\begin{equation}
\label{eq:robust_term}
r(t)=
\begin{cases}
-\rho(t)\frac{B^TPe(t)}{\|B^TPe(t)\|} \text{~~~~$\|B^TPe(t)\|>\epsilon$,}
\\
-\rho(t)\frac{B^TPe(t)}{\epsilon} \text{~~~~~~$\|B^TPe(t)\|\leq \epsilon$,}
\end{cases}
\end{equation}       
where $\rho(t)\in \RR$ is the last received upper bound on $\|\eta\|$ from the GPs, i.e., we use 
\begin{equation}
\label{rho_zoh}
\rho(t)=\rho_k, \forall t\in [kT_s,(k+1)T_s), 
\end{equation}
and $\epsilon$ is a small positive number. It should be noted that $\epsilon$ is a design parameter that can be selected to ensure high-accuracy tracking, as we will discuss in the next section.
%\end{itemize}

\section{Theoretical Guarantees}
\label{sec:theory}
After discussing the proposed strategy, we now justify that it satisfies both robust stability and high-accuracy tracking. To that end, we require the following reasonable assumption:

%\begin{assumption}
%\label{assum1}
%The uncertainty vector $\eta(q,\dot{q},a_q)$ in \eqref{eq:lin_sys_prac} can be modeled by $N$ independent GPs. %, and the used GP regression models are well tuned.   
%\end{assumption}

%The assumption says that conditioned on the $n$ observations of $(a_q,q,\dot{q})$, knowing an element $\eta_i$ does not provide any additional information for estimating $\eta_j$, for $i\neq j$.  

\begin{assumption}
\label{assum2}
The GPs run at a sufficiently fast sampling rate such that the calculated upper bound on $\|\eta\|$ is accurate between two consecutive sampling instants.  
\end{assumption}

We impose another assumption to ensure that the added robustness vector $r(t)$ will not cause the uncertainty vector norm $\|\eta(q(t),\dot{q}(t),a_q(t))\|$ to blow up. It is easy to show that the uncertainty function $\eta(q,\dot{q},a_q)$ is smooth, and so $\|\eta\|$ attains a maximum value on any compact set in its input space $(q,\dot{q},a_q)$. However, since from \eqref{eq:outer_controller} and \eqref{eq:robust_term}, $a_q$ is a function of $\rho(t)$, an upper bound on $\|\eta\|$, one still needs to ensure the boundedness of $\|\eta\|$ for bounded $q,~\dot{q}$ or bounded tracking error $e$. Hence, we present the following assumption.  

\begin{assumption}
\label{assum3}
For a given, smooth, bounded desired trajectory $(q_d(t),\dot{q}_d(t))$, there exists $\bar{\rho}>0$ such that $\|\eta\|\leq \bar{\rho}$ for each $e\in D$, where $D$ is a compact set containing $\{e\in \RR^{2N}:e^TPe\leq e(0)^TPe(0)\}$, and $e(0)$ is the initial tracking error.   
%For a smooth, bounded desired trajectory $(q_d(t),\dot{q}_d(t))$, if the tracking error $e\in D$, where $D$ is a compact set, then there exists a fixed upper bound $\rho$ 
%Given a smooth, bounded desired trajectory $(q_d(t),\dot{q}_d(t))$, we assume that there exists a fixed upper bound $\bar{\rho}>0$ such that $\|\eta\|$ in \eqref{eq:lin_sys_prac} satisfies $\|\eta\|<\bar{\rho}$ for 
\end{assumption}
%We now justify that Assumption \ref{assum3} is reasonable, and that it is automatically satisfied for small uncertainties in the system matrices, particularly in the inertia matrix $M(q)$.

We now justify that Assumption \ref{assum3} is reasonable. In particular, we show that the assumption is satisfied for small uncertainties in the inertia matrix $M(q)$ \cite{spong}.  
%and that it is automatically satisfied for small uncertainties in the system matrices, particularly in the inertia matrix $M(q)$. 
In this discussion, we suppose that $\rho(t)$ in \eqref{eq:robust_term} satisfies $\rho(t)\leq \|\eta(t)\|+c$, where $c$ is a positive scalar, and study whether imposing $r(t)$ into \eqref{eq:outer_controller} can make $\|\eta(t)\|$ blow up. Recall that $\eta(q,\dot{q},a_q)=M^{-1}(q)(\tilde{M}(q)a_q+\tilde{C}(q,\dot{q})\dot{q}+\tilde{g}(q))$. From \eqref{eq:outer_controller}, we have $\eta(q,\dot{q},a_q)=M^{-1}(q)(\tilde{M}(q)r+\tilde{M}(q)(\ddot{q}_d+K_P(q_d-q)+K_D(\dot{q}_d-\dot{q}))+\tilde{C}(q,\dot{q})\dot{q}+\tilde{g}(q))$. It is evident that $\|\eta\|\leq \|M^{-1}(q)\tilde{M}(q)\|\|r\|+\|T_b(q,\dot{q},q_d,\dot{q}_d,\ddot{q}_d)\|$, where $T_b=M^{-1}(q)(\tilde{M}(q)(\ddot{q}_d+K_P(q_d-q)+K_D(\dot{q}_d-\dot{q}))+\tilde{C}(q,\dot{q})\dot{q}+\tilde{g}(q))$. From \eqref{eq:robust_term}, it is easy to verify $\|r(t)\|\leq \|\rho(t)\|=\rho(t)$, and so $\|r(t)\|\leq \|\eta(t)\|+c$. Hence, $\|\eta\|\leq \|M^{-1}(q)\tilde{M}(q)\|\|\eta\|+ \|M^{-1}(q)\tilde{M}(q)\|c+\|T_b\|$. Now if the uncertainty in the matrix $M(q)$, $\tilde{M}(q)$, is sufficiently small such that $\|M^{-1}(q)\tilde{M}(q)\|<1$ is satisfied, then $\|\eta\| \leq \frac{1}{1-\|M^{-1}(q)\tilde{M}(q)\|}(\|M^{-1}(q)\tilde{M}(q)\|c+\|T_b(q,\dot{q},q_d,\dot{q}_d,\ddot{q}_d\|)$. 
%\begin{align*}
%\|\eta\| &\leq \frac{1}{1-\|M^{-1}(q)\tilde{M}(q)\|}(\|M^{-1}(q)\tilde{M}(q)\|c\\&+T_b(q,\dot{q},q_d,\dot{q}_d,\ddot{q}_d\|).
%\end{align*}

%$\|\eta\|\leq \frac{1}{1-\|M^{-1}(q)\tilde{M}(q)\|}(\|M^{-1}(q)\tilde{M}(q)\|c+\|M^{-1}(q)(\tilde{M}(q)(\ddot{q}_d+K_P(q_d-q)+K_D(\dot{q}_d-\dot{q}))+\tilde{C}(q,\dot{q})\dot{q}+\tilde{g}(q))\|)$. 
Since $q_d,~\dot{q}_d,~\ddot{q}_d$ are all bounded by assumption, if $e(t)\in D$, a compact set, then $q(t),~\dot{q}(t)$ are also bounded. It is easy to show that there exists a fixed upper bound $\bar{\rho}$ on $\|\eta\|$ that is valid for each $e\in D$, and Assumption \ref{assum3} is satisfied. 

\begin{remark} 
We have shown that if the uncertainty in the matrix $M(q)$, $\tilde{M}(q)$, is sufficiently small such that $\|M^{-1}(q)\tilde{M}(q)\|<1$ is satisfied, then Assumption \ref{assum3} holds. This argument is true even if we have large uncertainties in the other system matrices, $C(q,\dot{q})$ and $g(q)$. As indicated in Chapter 8 of \cite{spong}, if the bounds on $\|M\|$ are known ($\underline{m} \leq \|M\| \leq \overline{m}$), then one can always select $\hat{M}$ such that $\|M^{-1}(q)\tilde{M}(q)\|<1$ is satisfied. In particular, by selecting $\hat{M}=\frac{\overline{m}+\underline{m}}{2}I$, where $I$ is the identity matrix, it can be shown that $\|M^{-1}(q)\tilde{M}(q)\|\leq \frac{\overline{m}-\underline{m}}{\overline{m}+\underline{m}}<1$. Consequently, it is not difficult to satisfy the condition $\|M^{-1}(q)\tilde{M}(q)\|<1$ in practice, and Assumption \ref{assum3} is not restrictive.
\end{remark}

%Now on the one hand, and as discussed in the previous section, $\beta_k$ defining the confidence region of a GP estimate will inevitably grow to infinity as $t\rightarrow \infty$, which prevents us from providing any asymptotic guarantees of the system behavior as $t\rightarrow \infty$. On the other hand, 
From Assumption \ref{assum3}, we know that $\|\eta(t)\|\leq \bar{\rho}$ if $e(t) \in D$, and consequently, it is reasonable to saturate any estimate of $\rho(t)$ beyond $\bar{\rho}$. Hence, we suppose that the estimation of $\rho$ is slightly modified to be
\begin{equation}
\label{eq:rho_new}
\rho(t)=\min(\rho_{GP}(t),\bar{\rho}),
\end{equation}
where $\rho_{GP}(t)$ is the upper bound on the uncertainty norm, $\|\eta\|$, calculated from the GPs in \eqref{eq_rho_i}, \eqref{eq_rho}, and \eqref{rho_zoh}. It is straightforward to show that with the choice of $\rho(t)$ in \eqref{eq:rho_new} and for bounded smooth trajectories, the condition $e(t)\in D$ for all $t\geq 0$ implies that $a_q(t)$ in \eqref{eq:outer_controller} is always bounded, and so $a_{aug}=(q,\dot{q},a_q)$ always lies in a compact set. To be able to provide theoretical guarantees, we also assume w.l.o.g. that the small positive number $\epsilon$ in \eqref{eq:robust_term} is selected sufficiently small such that
\begin{equation}
\label{eq:epsilon}
\sqrt{\frac{\epsilon \bar{\rho}}{2\lambda_{min}(Q)}}\ll \delta_1,
\end{equation}
where $\delta_1>0$ is such that $B_{\delta_1}(0)\subset \{e\in \RR^{2N}:e^TPe<e(0)^TPe(0)\}$, and $\lambda_{min}(Q)>0$ is the smallest eigenvalue of the positive definite matrix $Q$.

Based on Assumptions \ref{assum0}, \ref{assum2} and \ref{assum3}, we provide the following main result.  

%Next, we impose another assumption to ensure that the added robustness term $r(t)$ will not cause the uncertainty vector norm $\|\eta\|$ to blow up. It is easy to show that the uncertainty function $\eta(q,\dot{q},a_q)$ is smooth, and so $\|\eta\|$ attains a maximum value on any compact set in its input space $(q,\dot{q},a_q)$. However, since from \eqref{eq:outer_controller} and \eqref{eq:robust_term}, $a_q$ is a function of $\rho(t)$, an upper bound on $\|\eta\|$, one still needs to prove the boundedness of $\|\eta\|$ for bounded $q,~\dot{q}$ or bounded tracking error.   
%
%\begin{assumption}
%\label{assum3}
%The uncertainty in the inertia matrix $M(q)$, $\tilde{M}$, is small enough such that $\|M^{-1}(q)\tilde{M}(q)\|<1$.  
%\end{assumption}
%
%\begin{lemma}
%\label{lem1}
%Consider the Lagrangian System \eqref{eq:sys_man}, and suppose that Assumption \ref{assum3} holds. Suppose that the control strategy in \eqref{eq:inv_prac}, \eqref{eq:outer_controller}, and \eqref{eq:robust_term} is applied, and that in \cite{}, 
%\end{lemma}

\begin{theorem}
\label{thm:main}
Consider the Lagrangian system \eqref{eq:sys_man} and a smooth, bounded desired trajectory $(q_d(t),\dot{q}_d(t))$. Suppose that Assumptions \ref{assum0}, \ref{assum2}, and \ref{assum3} hold. Then, the proposed, robust, learning-based control strategy in \eqref{eq:inv_prac}, \eqref{eq:outer_controller}, and \eqref{eq:robust_term}, with the uncertainty upper bound $\rho$ calculated by \eqref{eq:rho_new} and the design parameter $\epsilon$ satisfying \eqref{eq:epsilon}, ensures with high probability of at least $(1-\delta_p)^N$ that the tracking error $e(t)$ is uniformly ultimately bounded with an ultimate bound that can be made arbitrarily small through the selection of the design parameter $\epsilon$.
\end{theorem}
\begin{proof}
From Assumption \ref{assum3}, we know that $\|\eta(t)\|\leq \bar{\rho}$ when $e(t)\in D$, where $D$ is a compact set containing $\{e\in \RR^{2N}:e^TPe\leq e(0)^TPe(0)\}$. In the first part of the proof, we assume that the upper bound $\rho_{GP}(t)$ calculated by \eqref{eq_rho_i}, \eqref{eq_rho} and \eqref{rho_zoh} is a correct upper bound on $\|\eta(t)\|$ when $e(t)\in D$. Thus, in the first part of the proof, we know that $\rho(t)$ calculated by \eqref{eq:rho_new} is a correct upper bound on $\|\eta(t)\|$ when $e(t)\in D$, and we use Lyapunov stability analysis to prove that $e(t)$ is uniformly ultimately bounded. Then, in the second part of the proof, we use Lemma \ref{lemma1} 
%{\bf \emph{(existing Lemma that Felix told us about; remember that as Felix said the probability is high since the confidence region and hence $\rho_{GP}(t)$ will grow with time but this is fine since $\rho(t)$ will remain bounded from \eqref{eq:rho_new})}} 
to evaluate the probability of satisfying the assumption that $\rho_{GP}(t)$ is a correct upper bound on $\|\eta(t)\|$ when $e(t)\in D$, and hence, the probability that the provided guarantees hold.    
%on the probability of the confidence of the GPs to derive a probability on the satisfaction of the assumption that $\|eta\|<$
%prove that $\rho(t)$ is a correct upper bound with high probability $()^N$.

The first part of the proof closely follows the proof of the effectiveness of the robust controller in Theorem 3 of Chapter 8 of \cite{spong}, and we include the main steps of the proof here for convenience. Consider a candidate Lyapunov function $V(e)=e^TPe$. From \eqref{eq:tracking_error}, it can be shown that $\dot{V}=-e^TQe+2w^T(\eta+r)$, where $w=B^TPe$. Then, from \eqref{eq:robust_term}, we need to study two cases. 

For the case where $\|w\|>\epsilon$, we have 
\begin{align*}
w^T(\eta+r)&=w^T(\eta-\rho\frac{w}{\|w\|})=w^T\eta-\rho\|w\| \\ &\leq \|\eta\|\|w\|-\rho\|w\| 
\end{align*}
%$w^T(\eta+r)$$=$$w^T(\eta-\rho\frac{w}{\|w\|})$$=$$w^T\eta-\rho\|w\|\leq$$\|\eta\|\|w\|-\rho\|w\|$
from the Cauchy-Schwartz inequality. Since $\{e\in \RR^{2N}:e^TPe\leq e(0)^TPe(0)\}\subset D$ by definition and from Assumption \ref{assum3}, we know that $\|\eta\|\leq \bar{\rho}$. Also, by our assumption in this part of the proof, $\|\eta\|\leq \rho_{GP}$. Then, from \eqref{eq:rho_new}, $\|\eta\|\leq \rho$, and $w^T(\eta+r)\leq 0$. Thus, for this case, $\dot{V}\leq -e^TQe$, which ensures exponential decrease of the Lyapunov function. 

Next, consider the case where $\|w\|\leq \epsilon$. If $w=0$, then $\dot{V}=-e^TQe<0$. Then, for $\|w\|\leq \epsilon$ and $w\neq 0$, it is easy to show 
\[
\dot{V}=-e^TQe+2w^T(\eta+r)\leq-e^TQe+2w^T(\rho\frac{w}{\|w\|}+r).
\] 
From \eqref{eq:robust_term}, we have 
\[
\dot{V}\leq-e^TQe+2w^T(\rho\frac{w}{\|w\|}-\rho\frac{w}{\epsilon}). 
\]
It can be shown that the term $2w^T(\rho\frac{w}{\|w\|}-\rho\frac{w}{\epsilon})$ has a maximum value of $(\epsilon \rho)/2$ when $\|w\|=\epsilon/2$. Thus, $\dot{V}\leq -e^TQe+(\epsilon \rho)/2$. From \eqref{eq:rho_new}, $\rho\leq \bar{\rho}$, and consequently $\dot{V}\leq -e^TQe+(\epsilon \bar{\rho})/2$. If the condition $e^TQe>(\epsilon \bar{\rho})/2$ is satisfied, then $\dot{V}<0$. Since $Q$ is positive definite by definition, then $e^TQe\geq \lambda_{min}(Q)\|e\|^2$, where $\lambda_{min}(Q)>0$ is the smallest eigenvalue of $Q$. Hence, if $\lambda_{min}(Q)\|e\|^2>(\epsilon \bar{\rho})/2$, then $\dot{V}<0$. Thus, the Lyapunov function is strictly decreasing if $\|e\|>\sqrt{\frac{\epsilon \bar{\rho}}{2\lambda_{min}(Q)}}$. 
Let $B_{\delta}$ be the ball around the origin of radius $\delta:=\sqrt{\frac{\epsilon \bar{\rho}}{2\lambda_{min}(Q)}}$, $S_\delta$ be a sufficiently small sublevel set of the Lyapunov function V satisfying $\bar{B}_\delta\subset S_\delta^{\circ}$, and $B_c$ be the smallest ball around the origin satisfying $S_\delta\subset \bar{B}_c$. Since the Lyapunov function $V$ is strictly decreasing outside $\bar{B}_\delta$, the tracking error $e(t)$ eventually reaches and remains in $S_{\delta}\subset \bar{B}_c$, and so the tracking error $e(t)$ is uniformly ultimately bounded, and its ultimate bound is the radius of $B_c$. Note that from~\eqref{eq:epsilon}, $B_{\delta}\subset \{e\in \RR^{2N}:e^TPe<e(0)^TPe(0)\}\subset D$, and $\rho$ is a correct upper bound on $\|\eta\|$. One can see that $\delta$ and hence the radius of $B_c$ depend on the choice of the design parameter $\epsilon$. Indeed, $\epsilon$ can be selected sufficiently small to make $B_{\delta}$ and $B_c$ arbitrarily small. 
%Let $B_{\delta}$ be the ball around the origin of radius $\delta:=\sqrt{\frac{\epsilon \bar{\rho}}{2\lambda_{min}(Q)}}$, $S_\delta$ be the smallest sublevel set of the Lyapunov function V containing $B_\delta$, and $B_c$ be a sufficiently small ball around the origin containing $S_\delta$. Since the Lyapunov function $V$ is strictly decreasing outside $\bar{B}_\delta$, the tracking error $e(t)$ eventually reach and remain in $B_c$, and so the tracking error $e(t)$ is uniformly ultimately bounded, and its ultimate bound is the radius of $B_c$. Note that from~\eqref{eq:epsilon}, $B_{\delta}\subset \{e\in \RR^{2N}:e^TPe<e(0)^TPe(0)\}\subset D$, and $\rho$ is a correct upper bound of $\|\eta\|$. One can see that $\delta$ and hence the radius of $B_c$ depend on the choice of the design parameter $\epsilon$. Indeed, $\epsilon$ can be selected sufficiently small to make $B_{\delta}$ and $B_c$ arbitrarily small.              

In the second part of the proof, we calculate the probability of our assumption in the first part that $\rho_{GP}(t)$ is a correct upper bound on $\|\eta(t)\|$ when $e(t)\in D$. Recall that $e(t) \in D$ implies that $a_{aug}(t)$ is in a compact set, as discussed immediately after \eqref{eq:rho_new}.
%Recall that from our discussion immediately after \eqref{eq:rho_new}, $e(t)\in D$ for all $t\geq 0$ implies that $a_{aug}$ remains in a compact set. 
From Assumption \ref{assum2}, our problem reduces to calculating the probability that $\rho_{GP}$ is a correct upper bound on $\|\eta\|$ for all the sampling instants. 
%Then, from the way  $\rho_{GP}$ is calculated in \eqref{eq:rho_i} and \eqref{eq:rho}, this again reduces to studying the the probability that the confidence regions provided by all the GPs are all correct.  By Assumption \ref{assum0},   
%from Assumption \ref{assum1}, this again reduces to studying the probability that the confidence regions provided by the GPs are all correct. 
Using the confidence region proposed in Lemma \ref{lemma1} for calculating the upper bound on the absolute value of each element of $\eta$, and under Assumption \ref{assum0}, the probability that this upper bound is correct for all samples is higher than $(1-\delta_p)$ from Lemma \ref{lemma1}. 
Since the $N$ GPs are independent and the added noise to the output observations ${\bf \omega}_v$ is uncorrelated, then the probability that the upper bounds on the absolute values of all the elments of $\eta$, and hence the upper bound on $\|\eta(t)\|$, are correct is higher than $(1-\delta_p)^N$.    
\end{proof}

\begin{remark}
Although in practice it is difficult to estimate the upper bound $\bar{\rho}$ on $\|\eta\|$ used in \eqref{eq:rho_new}, one can be conservative in this choice. Unlike robust control techniques that keep this conservative bound unchanged, \eqref{eq:rho_new} would relax the upper bound $\bar{\rho}$ when the GPs learn a lower upper bound from collected data. Having a less-conservative upper bound $\rho$ results in a lower tracking error. It can be shown that if $\rho(t)\leq \rho'<\bar{\rho}$ for all $t$, then the tracking error will converge to an ultimate ball $B_{c'}$ smaller than $B_c$. 
\end{remark}

\begin{remark}
In theory, $\epsilon$ can be selected sufficiently small to ensure arbitrarily accurate tracking as shown in the proof of Theorem \ref{thm:main}. Achieving that for cases with large uncertainties may be limited by the actuation limits of the robots. Incorporating the actuation limits in the theoretical analysis is an interesting point for future research.
%While in theory $\epsilon$ can be selected sufficiently small to ensure arbitrarily accurate tracking as shown in the proof of Theorem \ref{thm:main}, in practice, having a smaller $\epsilon$ will result in a higher $r(t)$ from \eqref{eq:robust_term}, a higher $a_q(t)$, and a higher control action $u$. Thus, achieving arbitrarily accurate tracking may be limited by the actuation limits of the robots, which is a typical scenario in practice. Incorporating the actuation limits in the theoretical analysis is an interesting point for future research.   
\end{remark}

\section{Simulation Results} \label{sec:sim}
The proposed approach is first verified via simulations on a 2 DOF planar manipulator using MATLAB Simulink 
%Several simulations on a 2 DOF planar manipulator are performed using MATLAB Simulink to demonstrate the effectiveness of our approach.

%\subsection{Simulation Setup}
We use the robot dynamics \eqref{eq:sys_man} for the system, where $M$, $C$,
and $g$ are as defined in Chapter 7 of \cite{spong}. For the system parameters,
a value of $1~\mathrm{kg}$ is used for each link mass and $1~
\mathrm{kg\cdot{m^2}}$ for each link inertia. The length of the first link is
$2~ \mathrm{m}$ and that of the second link is $1~\mathrm{m}$. The joints are
assumed to have no mass and are not affected by friction. Then, it is assumed
that these parameters are not perfectly known. Thus, in the inverse dynamics
controller \eqref{eq:inv_prac}, we use parameters with different levels of
uncertainties. 
%, and for each uncertainty case, evaluate our proposed control
%strategy. 
The desired trajectories are sinusoidal trajectories with different
amplitudes and frequencies. All the simulation runs are initialized at zero
initial conditions.

We use $2$ GPs to learn the uncertainty vector $\eta$ in
\eqref{eq:lin_sys_prac}. Each GP uses the squared exponential kernel
parameterized with $\sigma_{\eta,{i}}=1$, $\sigma_{\omega,{i}}=0.001$, and
$l_{j,i}=0.5$, for all $j\in\{1,\cdots,6\}$ and $i\in \{1,2\}$. The GPs run at $10~\mathrm{Hz}$ and use the
past $n=20$ observations for prediction. 
To generate confidence intervals, we use $[\mu_k-3\sigma_k,\mu_k+3\sigma_k]$,
which is simple to implement and found to be effective in practice
\cite{GP_felix}. For the robust controller, we use
$\epsilon=0.001$. We set the upper bound $\bar{\rho}$ in \eqref{eq:rho_new}
to be a very high positive number to evaluate the effectiveness of the upper bound estimated by the GPs.
%\subsection{Results}

A sequence of \emph{12 trajectories} was run for \emph{3 different cases of
model uncertainty}. Each of the three cases makes the $\hat{M}$ matrix differ
from the $M$ matrix by using values for the estimated link masses that differ
from the true link mass values. In particular, in the three uncertainty cases,
the estimated mass differs from the actual mass by $10\%$, $20\%$, and $30\%$
for each link, respectively.

The tracking performance was compared between four controllers: a nominal
controller with no robust control, a robust controller with a fixed upper bound
on the uncertainty norm $\rho=1000$, a learning-based inverse dynamics controller in which GPs are used to learn the error of the nominal inverse model at the torque level $\Delta u$ and a non-robust outer-loop controller is used, and our proposed robust learning controller. The root-mean-square (RMS) error of the joint angles was averaged
over the 12 trajectories, and is presented for each controller and uncertainty
case in Table \ref{table:2dof_rmse}.
%The tracking performance was compared between three controllers: a nominal
%controller with no robust control, a robust controller with a fixed upper bound
%on the uncertainty norm $\rho=1000$, and our proposed, robust, learning-based
%controller. The root-mean-square (RMS) error of the joint angles was averaged
%over the 12 trajectories, and is presented for each controller and uncertainty
%case in Table \ref{table:2dof_rmse}.
\begin{table}[t]
 \vspace*{-1mm}
  \renewcommand{\arraystretch}{1.1}
  \footnotesize
  \caption{Average RMS Tracking Error (in rad) Over 12 Trajectories for Different Controllers on a 2 DOF Manipulator}
  \label{table:2dof_rmse}
  \centering
   \vspace*{-3mm}
    \begin{tabular}{ c | c c c c }
      \hline
      Uncertainty & Nominal & \thead{Fixed\\Robust} & Learning $\Delta u$ & \thead{{\bf Robust}\\{\bf Learning}}\\
      \hline
      $10\%$ & 0.1554 & 0.0476 & 0.0190 & 0.0082\\
      $20\%$ & 0.2793 & 0.0498 & 0.0319 & 0.0103\\
      $30\%$ & 0.3768 & 0.0519 & 0.0539 & 0.0141\\
      \hline
    \end{tabular}
\end{table}
\normalsize

%\begin{table}[t]
%  \renewcommand{\arraystretch}{1.0}
%  \caption{Average RMS Tracking Error (in rad) Over 12 Trajectories for Different Controllers on a 2 DOF Manipulator}
%  \label{table:2dof_rmse}
%  \centering
%    \begin{tabular}{ c | c c c }
%      \hline
%      Uncertainty & Nominal & Fixed Robust & Robust Learning\\
%      \hline
%      $10\%$ & 0.1557 & 0.0489 & 0.0065\\
%      $20\%$ & 0.2807 & 0.0512 & 0.0101\\
%      $30\%$ & 0.3793 & 0.0534 & 0.0148\\
%      \hline
%    \end{tabular}
%\end{table}

It is clear that while the robust controller with a high, fixed value for the
upper bound on the uncertainty improves the tracking performance compared to
the nominal controller, it is conservative, and thus, still causes considerable
tracking errors. The tracking errors are significantly reduced by our proposed
robust learning controller, which is able to learn a less conservative
upper bound on the uncertainty. On average, our proposed controller
reduces the tracking errors by $95.8\%$ compared to the nominal controller,
by $78.2\%$ compared to the fixed, robust controller, and by $66\%$ compared to the non-robust learning controller that learns $\Delta u$.

\section{Experimental Results} \label{sec:exp}

The proposed approach is further tested on a UR10 6~DOF industrial
manipulator (see Figure \ref{fig:ur10}) using the Robot Operating System (ROS).

\begin{figure}[t]
  \begin{center}
    \includegraphics[width=1.7in,trim = 0mm 97mm 0mm 50mm]{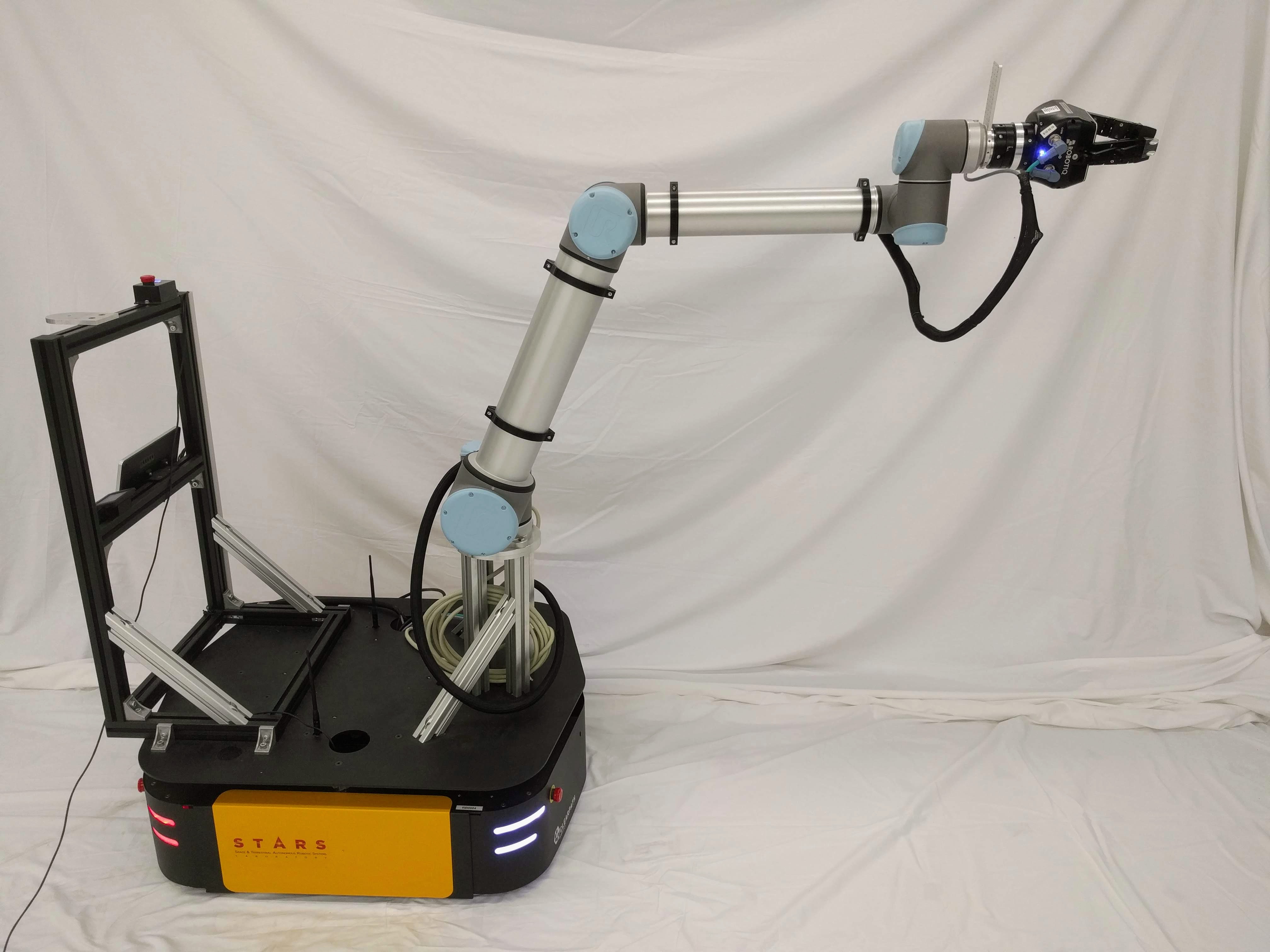}
  \end{center}
  \caption{The UR10 industrial manipulator used in the experiments.}
  \label{fig:ur10}
\end{figure}

\subsection{Experimental Setup}
The interface to the UR10 does not permit direct torque control. Instead, only
position and velocity control of the joints are available. Thus, for our proposed approach, we need to implement only the GP regression models and the robust, outer-loop controller. The commanded
acceleration $a_q$ calculated by the outer-loop controller in
\eqref{eq:outer_controller} is integrated to obtain a velocity command that can
be sent to the UR10. To test our approach for various uncertainties, we
introduce artificial model uncertainty by adding a function $\eta(q, \dot{q},
a_q)$ to our calculated acceleration command $a_q$.
%The dynamic and kinematic parameters of the system model are provided by the UR10
%robot manipulator's open-source Gazebo package in \cite{ROS}. 

The PD gains of the outer-loop controller are tuned to achieve good tracking performance on
a baseline desired trajectory in a nominal scenario with no added uncertainty. A desired trajectory of
$q_d=0.25(1-\cos(2.0t))$ for each joint is used for this purpose, with gains
selected to produce a total joint angle RMS error less than
$0.01~\mathrm{rad}$. This resulted in $K_p=7.0I$ and
$K_d=1.0I$, where $I$ is the identity matrix.

We use 6 GPs to learn the uncertainty vector~$\eta$, each of which uses the
squared exponential kernel. The prior variance and length scale hyperparameters
are optimized by maximizing the marginal likelihood function, while each noise
variance is set to $0.001$. Hyperparameter optimization is performed offline
using approximately 1000 data points collected while tracking sinusoidal trajectories under uncertainty $\eta=0.5\dot{q}$. Implementation and
tuning of the GPs are done with the Python library GPy. Each GP runs at $125~\mathrm{Hz}$, and uses the
past $n=50$ observations for prediction. For the confidence intervals, we use $[\mu_k-3\sigma_k,\mu_k+3\sigma_k]$ for simplicity \cite{GP_felix}. 
%with predictions running at a sample
%rate of $125~\mathrm{Hz}$. 
For the robust controller, we use $\epsilon=0.1$.

%The interface to the UR10 does not permit direct torque control. Instead, only
%position and velocity control of the joints are available. Thus, the commanded
%acceleration $a_q$ calculated by the outer-loop controller in
%\eqref{eq:outer_controller} is integrated to obtain a velocity command that can
%be sent to the UR10. To test our approach for various uncertainties, we
%introduce artificial model uncertainty by adding a function $\eta(q, \dot{q},
%a_q)$ to our calculated acceleration command $a_q$.

\subsection{Results}

The performance of the proposed robust learning controller is initially
compared to that of the nominal, outer-loop PD controller using a single
trajectory and various cases of model uncertainty. Ten different cases of
uncertainty of the form $\eta(q, \dot{q}, a_q)$ are tested over
the desired trajectory $q_d=0.25(1-\cos(2.0t))$ for each joint, with the
results displayed in Figure \ref{fig:many_eta}. The average RMS error of the
nominal controller is $0.111~\mathrm{rad}$ and that of the proposed
controller is $0.068~\mathrm{rad}$, yielding an average improvement of
$41.5\%$.

\begin{figure}[t]
  \centering
  \includegraphics[width=2.2in, trim = 0mm 10mm 0mm 1mm]{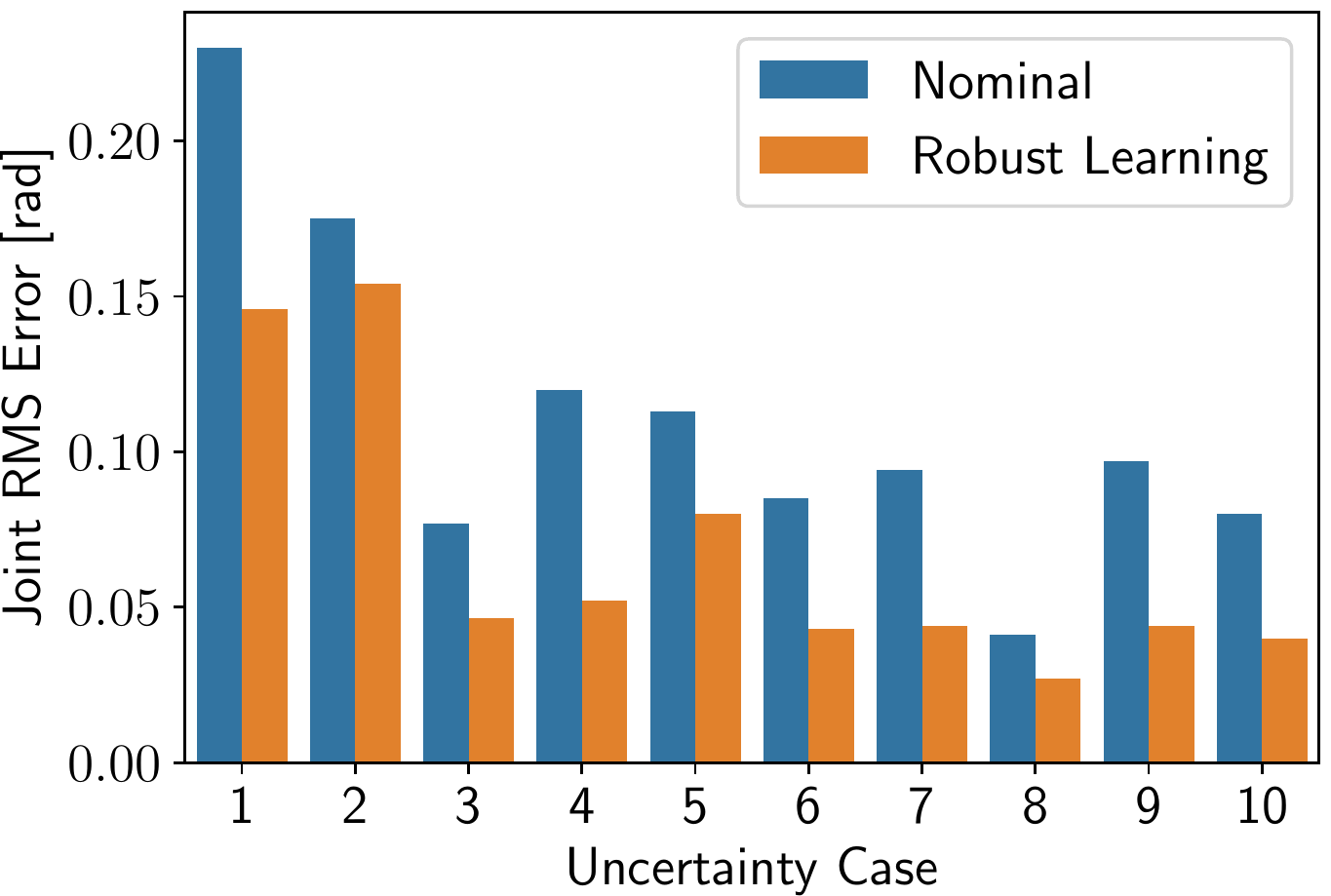}
  \caption{Joint RMS error of the nominal controller and the proposed robust
  learning controller for ten different uncertainties and a desired trajectory
  $q_d=0.25(1-\cos(2.0t))$ for each joint. The proposed method outperforms the
  nominal controller in all cases, with an average error reduction of
  $41.5\%$.}
  \label{fig:many_eta}
\end{figure}

Further experiments were performed to verify the generalizability of the proposed approach for different desired trajectories that cover different regions of
the state space. A single case of uncertainty, $\eta=0.3\dot{q}+0.01q\odot \dot{q}$ where $\odot$ is the entrywise product,   
is selected and the performance of the proposed and nominal controllers under
this uncertainty is compared on five additional trajectories. The results are
presented in Table \ref{table:many_traj}, with an average overall improvement
of $39.9\%$ compared to the nominal controller. The six trajectories are shown in a demo video at \url{http://tiny.cc/man-traj} 
%The end effector paths of the
%six trajectories are in Figure \ref{fig:trajs}.

%\begin{figure}[t]
%  \centering
%  \includegraphics[width=1.65in, trim = 0mm 12mm 0mm 15mm]{trajs.pdf}
%  \caption{The end effector position paths for trajectories from Table \ref{table:many_traj}.}
%  \label{fig:trajs}
%\end{figure}

\begin{table}[t]
 \vspace*{-2.5mm}
  \renewcommand{\arraystretch}{1.1}
\footnotesize
  \caption{RMS Tracking Error (in rad) for Various Trajectories with a Fixed
  Uncertainty Function $\eta=0.3\dot{q}+0.01q\odot \dot{q}$, where $\odot$ is the entrywise product}
  \label{table:many_traj}
  \centering
  \vspace*{-3mm}
    \begin{tabular}{ c | c c | c }
      \hline
      Trajectory & Nominal & {\bf Robust Learning} & Improvement\\
      \hline
%       1 & 0.085 & 0.043 & 49.4\%\\
%       2 & 0.058 & 0.037 & 36.2\%\\
%       3 & 0.070 & 0.037 & 47.1\%\\
%       4 & 0.092 & 0.058 & 37.0\%\\
%       5 & 0.029 & 0.021 & 27.6\%\\
%       6 & 0.050 & 0.029 & 42.0\%\\
       1 & 0.070 & 0.037 & 47.1\%\\
       2 & 0.058 & 0.037 & 36.2\%\\
       3 & 0.092 & 0.058 & 37.0\%\\
       4 & 0.085 & 0.043 & 49.4\%\\
       5 & 0.050 & 0.029 & 42.0\%\\
       6 & 0.029 & 0.021 & 27.6\%\\
      \hline
      Average & 0.064 & 0.038 & 39.9\%\\
      \hline
    \end{tabular}
 \vspace*{-4mm}
\end{table}
\normalsize

To verify the reliability of the proposed method, experiments for six
combinations of uncertainty and trajectory are repeated five times each with
both the nominal and proposed robust learning controllers. The results are
summarized in Figure \ref{fig:boxplot}. The figure shows that the performance
under our proposed controller is highly repeatable and that it outperforms the
nominal controller in all 30 cases.

\begin{figure}[t]
  \centering
  \includegraphics[width=2.1in, trim = 0mm 10mm 0mm 1mm]{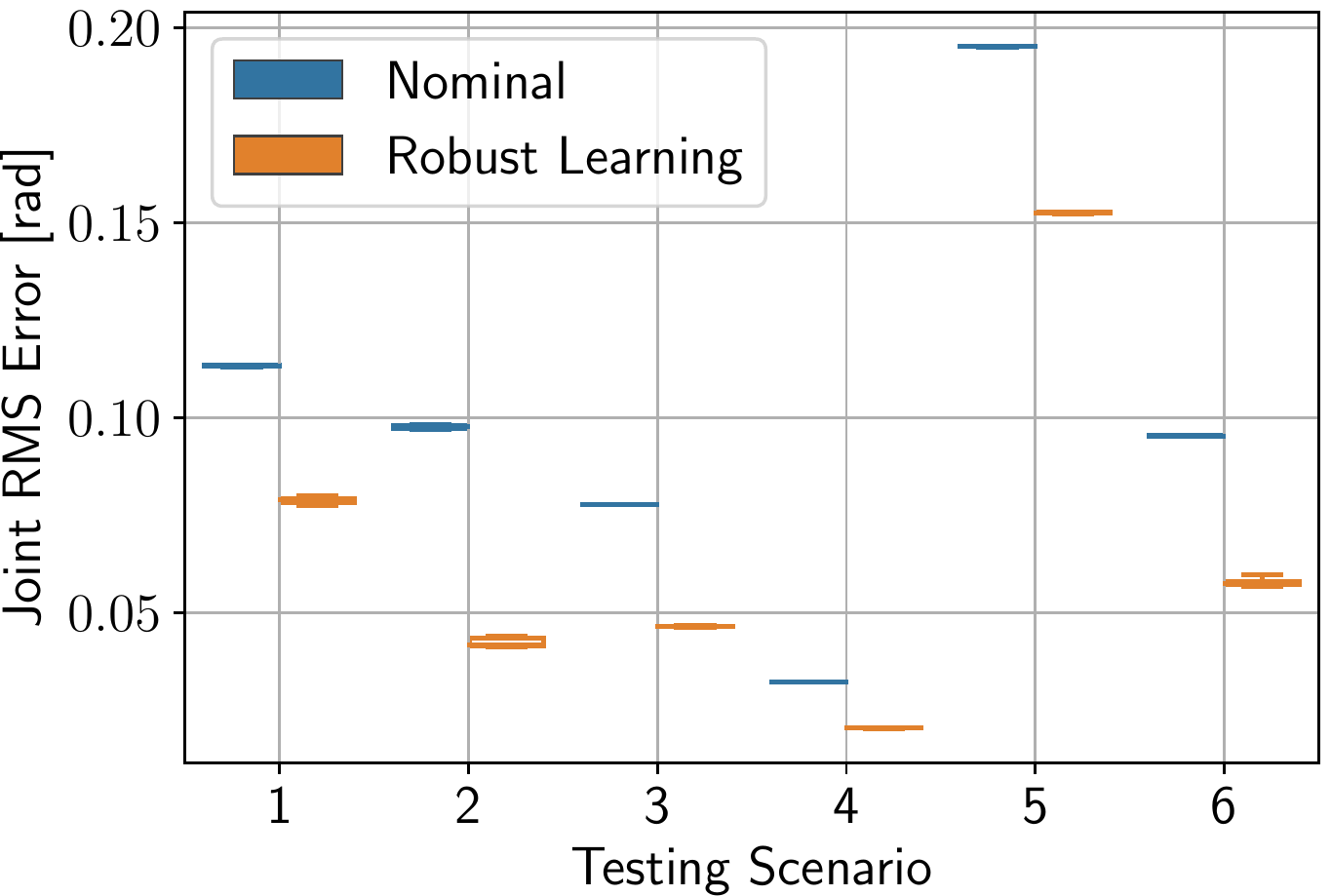}
  \caption{Box plots of six testing scenarios with different
  combinations of uncertainty and trajectory, each repeated five times. The
  bottom and top edges of each box show the 25\textsuperscript{th} and
  75\textsuperscript{th} percentiles, respectively. The whiskers show the
  extreme values. The system performance is very repeatable, and our proposed
  controller outperforms the nominal controller in all cases (30 experiments in total).}
  \label{fig:boxplot}
\end{figure}

\section{Conclusions}
\label{sec:con}
We have provided a novel, learning-based control strategy based on Gaussian processes (GPs) that ensures stability of the closed-loop system and high-accuracy tracking of smooth trajectories for an important class of Lagrangian systems. The main idea is to use GPs to estimate an upper bound on the uncertainty of the linearized model, and then use the uncertainty bound in a robust, outer-loop controller.
%In this strategy, we have used GPs to learn the uncertainty in the nonlinear system linearization, particularly, the error between the commanded and actual accelerations of the system. We have then used the mean and variance of the GPs predictions to calculate an upper bound on the uncertainty that is guaranteed to be correct with high probability. The calculated upper bound is then used in a robust, outer-loop controller. 
Unlike most of the existing, learning-based inverse dynamics control techniques, we have provided a proof of the closed-loop stability of the system that takes into consideration the regression errors of the learning module. Moreover, we have proved that the tracking error converges to a ball with a radius that can be made arbitrarily small. Furthermore, we have verified the effectiveness of our approach via simulations on a planar manipulator and experimentally on a 6 DOF industrial manipulator. 

\bibliographystyle{IEEEtranS}

\end{document}